\documentclass[11pt]{article}
\usepackage[bottom=1in,top=1in,left=1in,right=1in]{geometry}

\usepackage{setspace}
\usepackage{natbib}

\usepackage[utf8]{inputenc} 
\usepackage[T1]{fontenc}    
\usepackage{hyperref}       
\usepackage{url}            
\usepackage{booktabs}       
\usepackage{amsfonts}       
\usepackage{nicefrac}       
\usepackage{microtype}      
\usepackage{algorithm,algorithmic}
\usepackage{xcolor}         
\usepackage{multirow}
\usepackage{bbm}

\usepackage{amsmath,amscd,amsbsy,amssymb,latexsym,url,bm,amsthm}
\allowdisplaybreaks
\usepackage{cleveref}
\usepackage{verbatim}
\usepackage{color}
\usepackage{dsfont}

\newtheorem{corollary}{Corollary}
\newtheorem{theorem}{Theorem}

\newtheorem{lemma}{Lemma}

\newcommand{\cF}{\mathcal{F}}

\newcommand{\abs}[1]{\left| #1 \right|}
\newcommand{\bOne}[1]{\mathds{1} \! \left\{#1\right\}}

\newcommand{\EE}[1]{\mathbb{E} \left[#1\right]}
\newcommand{\PP}[1]{\mathbb{P} \left(#1\right)}

\usepackage{tikz}
\usetikzlibrary{shapes,shapes.misc,positioning}
\usetikzlibrary{patterns,arrows,decorations.pathreplacing}

\mathchardef\mhyphen="2D

\usepackage{hyperref}

\usepackage{times}
\usepackage{multirow}
 \usepackage{booktabs}
\usepackage{amsmath,amscd,amsbsy,amssymb,latexsym,url,bm,natbib}
\allowdisplaybreaks
\usepackage{cleveref}
\usepackage{verbatim}
\usepackage{color}
\usepackage{dsfont}
\usepackage{caption}
\usepackage{algorithm}
\usepackage{algorithmic}
\usepackage{threeparttable}
\usepackage{graphicx}
\usepackage{thmtools}
\usepackage{thm-restate}
\usepackage{mathtools}

\hypersetup{
  colorlinks   = true, 
  urlcolor     = blue, 
  linkcolor    = blue, 
  citecolor   = blue 
}

\title{
Optimal Analysis for Bandit Learning in Matching Markets with Serial Dictatorship
}

\author{
  Zilong Wang$^{1}$, Shuai Li$^{*1}$ \\
  $^{1}$Shanghai Jiao Tong University, \text{\{wangzilong,  shuaili8\}@sjtu.edu.cn} 
}
\date{}

\begin{document}
\maketitle
\begingroup
\renewcommand\thefootnote{*}
\footnotetext{Corresponding author: Shuai Li (\texttt{shuaili8@sjtu.edu.cn})}
\endgroup

\begin{abstract}
The problem of two-sided matching markets is well-studied in computer science and economics, owing to its diverse applications across numerous domains. Since market participants are usually uncertain about their preferences in various online matching platforms, an emerging line of research is dedicated to the online setting where one-side participants (players) learn their unknown preferences through multiple rounds of interactions with the other side (arms).
  \citet{sankararaman2021dominate} provide an $\Omega\left( \frac{N\log(T)}{\Delta^2} + \frac{K\log(T)}{\Delta} \right)$ regret lower bound for this problem under serial dictatorship assumption, where $N$ is the number of players, $K (\geq N)$ is the number of arms, $\Delta$ is the minimum reward gap across players and arms, and $T$ is the time horizon. Serial dictatorship assumes arms have the same preferences, which is common in reality when one side participants have a unified evaluation standard.
  Recently, the work of \citet{kong2023player} proposes the ET-GS algorithm and achieves an $O\left( \frac{K\log(T)}{\Delta^2} \right)$ regret upper bound, which is the best upper bound attained so far. Nonetheless, a gap between the lower and upper bounds, ranging from $N$ to $K$, persists. It remains unclear whether the lower bound or the upper bound needs to be improved. In this paper, we propose a  multi-level successive selection algorithm that obtains an $O\left( \frac{N\log(T)}{\Delta^2} + \frac{K\log(T)}{\Delta} \right)$ regret bound when the market satisfies serial dictatorship. To the best of our knowledge, we are the first to propose an algorithm that matches the lower bound in the problem of matching markets with bandits.
\end{abstract}
\section{Introduction}
The two-sided matching market is a classical model and has been widely studied \citep{gale1962college,roth1984evolution,roth1992two,roth2002economist}. It has various applications like school admission and labor markets. There are participants on two sides of the market, and the participant on each side has a preference over the opposite side. Stability is a critical property that characterizes the equilibrium state of the matching. Usually, the stable matching is not unique and how to find it in a given market has been studied for a long time \citep{gale1962college,roth1992two}. In these studies, preferences are assumed to be fixed and known. However, this assumption is often unrealistic in real-world applications, such as online labor markets where employers are uncertain about their preferences for workers and can learn through multiple rounds of iterative matching. Multi-armed bandit (MAB) serves as a crucial tool for making decisions under uncertainty. This framework typically involves a single player and $K$ arms, with each arm possessing a distinct reward distribution unknown to the player. The player's objective is to minimize their cumulative regret, defined as the expectation of the difference in cumulative rewards between the arm with the highest reward and the player's selected arm over $T$ rounds.

The problem of bandit learning in matching markets is first introduced by \citet{das2005two} and has been studied by a rich line of work \citep{liu2020competing,liu2021bandit,basu2021beyond,sankararaman2021dominate,kong2022Thompson,kong2023player,zhang2022matching,wang2022bandit}. Players and arms correspond to the participants on two sides of markets. Each player has an unknown preference, while arms are certain about their preferences over players. For example, workers' preferences are ranked by the job's salary, which is fixed and known. However, employers are uncertain about their preferences since workers' ability is unknown and can be learned through interactions. This problem aims to minimize the stable regret, defined as the reward difference between the stable matched arm and the player's selected arm over $T$ rounds \citep{liu2020competing}. When the stable matching is not unique, there are mainly two types of regret: one is the player-optimal stable regret concerning the player's most preferred stable matching; the other one is the player-pessimal stable regret with respect to the player's least preferred stable matching.

The work of \citet{liu2020competing} first theoretically studies a centralized setting with a platform to assign arms to each player. They propose the explore-then-commit (ETC) algorithm and show the $O\left( \frac{K\log(T)}{\Delta^2} \right)$ player-optimal regret. However, the algorithm requires the knowledge of minimum reward gap $\Delta$, which is hard to know in advance. They also propose the upper confidence bound (UCB) algorithm but can only achieve the player-pessimal stable regret. The following works study the more general decentralized setting \citep{liu2021bandit, sankararaman2021dominate,basu2021beyond,maheshwaridecentralized,kong2022Thompson,zhang2022matching, kong2023player}, where each player takes action based on its observed information independently. \citet{sankararaman2021dominate} show an $\Omega\left( \frac{N\log(T)}{\Delta^2} + \frac{K\log(T)}{\Delta} \right)$ stable regret lower bound. \citet{kong2023player} propose the ET-GS algorithm that can achieve $O\left(\frac{K\log(T)}{\Delta^2}\right)$ regret upper bound. This is the closest algorithm to the lower bound so far. It matches the lower bound with respect to $T$, but there is still a gap from $N$ to $K$. Note that all works mentioned above assume $N\leq K$ since each player needs to have a chance to match one arm in a round, and $N$ might be much smaller than $K$. Thus it is valuable to design an algorithm that achieves the lower bound.

\subsection{Our Contributions}
In this paper, we propose the multi-level successive selection algorithm that attains the $O\left( \frac{N\log(T)}{\Delta^2} + \frac{K\log(T)}{\Delta} \right)$ stable regret upper bound. This result matches the lower bound proposed in \citet{sankararaman2021dominate} under the same condition called serial dictatorship, where all arms share the same preference. This condition guarantees the uniqueness of stable matching, which is also studied by previous works \citep{sankararaman2021dominate,basu2021beyond,maheshwaridecentralized}.

Inspired by the hierarchical structure of players under the serial dictatorship assumption, we propose the decentralized multi-level successive selection algorithm that converges to stable matching from the most preferred player by arms to the least preferred one. The key idea is that the player with rank $i$ takes actions in the arm set except the arms players from $p_1$ to $p_{i-1}$ select. This is reasonable since the lower-ranked player has no chance of winning the competition with the higher-ranked player. Our work serves as an important first step, proving that the upper bound can be improved to match the lower bound in cases where the lower bound has been proved. This work also gives valuable intuition and direction for algorithm design for the general matching market.

\section{Related Work}
The study of bandit learning in matching markets is first introduced by \citet{das2005two}. They study the uncertainty of preferences under the special market structure with the same preferences on both sides. \citet{liu2020competing} first theoretically analyzed the online matching market with the MAB problem. They consider the centralized scenario, in which a platform receives the preferences of both sides and then runs the Gale-Shapley algorithm to output a stable matching. They apply both the ETC algorithm and the UCB algorithm to estimate the ranking. The former requires the knowledge of reward gap $\Delta$ and time horizon $T$, and it achieves the $O\left( {K\log(T)}/{\Delta^2} \right)$ player-optimal regret. They then prove the $O\left( {NK\log(T)}/{\Delta^2} \right)$ player-pessimal regret for the centralized UCB algorithm. Thereafter, \citet{liu2021bandit} and \citet{kong2022Thompson} propose the first UCB and TS-type algorithm for the decentralized market, respectively. They both obtain $O\left( {\exp(N^4) N^5K^2\log^2(T)}/{\Delta^2} \right)$ player-pessimal stable regret. The work of \citet{zhang2022matching} proposes the ML-ETC algorithm and achieves the $O\left({NK\log(T)}/{\Delta^2}\right)$ player-optimal stable regret bound. \citet{kong2023player} proposes the ETGS algorithm and improves the player-optimal stable regret bound to $O\left({K\log(T)}/{\Delta^2}\right)$. \citet{wang2022bandit} study the bandit learning in many-to-one matching markets, where an arm can select more than one player from the other side.

A line of research studies the special market that satisfies some conditions to guarantee unique stable matching. The work of \citet{sankararaman2021dominate} proposes the UCB-D3 algorithm based on the assumption of serial dictatorship, and they obtain an $O\left( {NK\log(T)}/{\Delta^2} \right)$ regret bound. They also derive an $\Omega\left(\max \left\{{N\log(T)}/{\Delta^2}, {K\log(T)}/{\Delta} \right\} \right)$ regret lower bound. The work of \citet{basu2021beyond} assumes the market satisfies $\alpha$-condition, which generalizes the serial dictatorship. They propose the UCB-D4 algorithm and also achieve the $O\left( {NK\log(T)}/{\Delta^2} \right)$ regret bound. \citet{maheshwaridecentralized} study the market satisfying $\alpha$-reducible and proposes a communication-free algorithm. In table \ref{tab:result} we compare our work with other related results. It can be seen that our algorithm is the first that match the lower bound.

There are also some works related to the problem of online matching markets with preference learning. \citet{wang2022bandit, kong2024compatibility} study the many-to-one matching markets where an arm can select more than one player. \citet{jagadeesan2021learning,cen2022regret} study the online matching markets with the monetary transfer. \citet{min2022learn} consider the Markov matching markets, where state transition will occur during the matching process, and the rewards obtained by players depend on the current state. Another line of work focuses on the non-stationary rewards for players in the matching market \citep{li2022rate,ghosh2022decentralized,muthirayan2022competing}. They design robust algorithms to reduce the impact of reward disturbance. There are also many studies on offline matching market learning \citep{dai2021learning,dai2021learningin,su2022optimizing}. They design the optimal matching based on historical data or recommend the participants on both sides. This kind of work does not involve the analysis of cumulative regret.

\begin{table}[htb]
\caption{Comparisons of settings and regret bounds with most related works. $N$ is the number of players, $K\geq N$ is the number of arms, $T$ is the time horizon, $\Delta$ is the minimum reward gap across all players and arms, $\epsilon$ depends on the hyper-parameter of algorithms, $C$ is related to the structure of markets and can grow exponentially in $N$, ``unique'' means that the stable matching is unique.} \label{tab:result}
\resizebox{\textwidth}{!}{\begin{tabular}{@{}lccccccc@{}}
\toprule
     &                                                                                             & \begin{tabular}[c]{@{}c@{}}Regret\\Bound\end{tabular}   & \begin{tabular}[c]{@{}c@{}}Regret\\Type\end{tabular}&\begin{tabular}[c]{@{}c@{}}Assumption\\on Preferences\end{tabular}                                                     \\ \midrule
    & \begin{tabular}[c]{@{}c@{}} \cite{liu2020competing} \end{tabular}       &  \begin{tabular}[c]{@{}c@{}} $O\left( \frac{K\log(T)}{\Delta^2} \right)$  \end{tabular}                                             & \begin{tabular}[c]{@{}c@{}} player-optimal \end{tabular}                                               &  \begin{tabular}[c]{@{}c@{}} general market\end{tabular}                                                \\        
    &        & $O\left( \frac{NK\log(T)}{\Delta^2} \right)$                                              & player-pessimal                                               &  general market                                                    \\   \midrule  
    & \begin{tabular}[c]{@{}c@{}} \cite{liu2021bandit} \end{tabular}       & \begin{tabular}[c]{@{}c@{}} $O\left( \frac{\exp(N^4) N^5K^2\log^2(T)}{\Delta^2} \right)$ \end{tabular}                                               & player-pessimal                                               &  general market                                                    \\ \midrule
     & \begin{tabular}[c]{@{}c@{}} \cite{kong2022Thompson} \\\end{tabular}       & $O\left( \frac{\exp(N^4) N^5K^2\log^2(T)}{\Delta^2} \right)$                                               & player-pessimal                                               &  general market                                                    \\  \midrule
     & \begin{tabular}[c]{@{}c@{}} \cite{sankararaman2021dominate} \\\end{tabular}       & $O\left(\frac{NK\log(T)}{\Delta^2}\right)$                                              & unique                                               &  serial dictatorship                                                    \\ 
      &      & $\Omega\left(\max \left\{\frac{N\log(T)}{\Delta^2}, \frac{K\log(T)}{\Delta} \right\} \right)$                                              & unique                                               &  serial dictatorship                                                    \\  \midrule
      & \begin{tabular}[c]{@{}c@{}} \cite{basu2021beyond} \\\end{tabular}       & $O\left(\frac{NK\log(T)}{\Delta^2}\right)$                                             & unique                                               &  $\alpha$-condition                                                   \\
      &   & $O\left(K\log^{1+\epsilon}(T) + 2^{(\frac{1}{\Delta^2})^{\frac{1}{\epsilon}}}\right)$                                              & player-optimal                                               &  general market                                                    \\ \midrule
      & \begin{tabular}[c]{@{}c@{}} \cite{maheshwaridecentralized} \\\end{tabular}       & $O\left(\frac{CNK\log(T)}{\Delta^2}\right)$                                            & unique                                               &  $\alpha$-reducible                                                    \\ \midrule
      & \begin{tabular}[c]{@{}c@{}} \cite{zhang2022matching} \\\end{tabular}       & $O\left(\frac{K\log(T)}{\Delta^2}\right)$                                            & player-optimal                                               &  general market                                                   \\ \midrule
      & \begin{tabular}[c]{@{}c@{}} \cite{kong2023player} \\\end{tabular}       & $O\left(\frac{K\log(T)}{\Delta^2}\right)$                                          & player-optimal                                               &  general market                                                  \\ \midrule
      & \begin{tabular}[c]{@{}c@{}} \textbf{ours}
      \\\end{tabular}       & $O\left(\frac{N\log(T)}{\Delta^2} + \frac{K\log(T)}{\Delta}\right)$                                              & unique                                               &  serial dictatorship                                                  \\
    \bottomrule
\end{tabular}}
\end{table}

\section{Preliminaries}
In this section, we introduce the model of bandit learning in matching markets. Suppose there are $N$ players and $K$ arms. Denote $\mathcal{N} = \{p_1,p_2,\cdots,p_N\}$ as the set of players and $\mathcal{K} = \{a_1,a_2,\cdots,a_K\}$ as the set of arms. To ensure no player will be unmatched, we assume $N\leq K$, similar to previous work \citep{basu2021beyond,kong2022Thompson,liu2020competing,liu2021bandit,sankararaman2021dominate,wang2022bandit}. Each arm has a fixed preference rank $\pi_{k,i}$ over players. $\pi_{k,i} > \pi_{k,i'}$ means arm $a_k$ prefers player $p_i$ to $p_{i'}$. The preference of player $p_i$ over arm $a_k$ is modeled by the utility $\mu_{i,k} > 0$. $\mu_{i,k}>\mu_{i,k'}$ implies that player $p_i$ prefers arm $a_k$ rather than $a_{k'}$. As in previous work, we assume all preferences are distinct, i.e., $\mu_{i,k}\neq \mu_{i,k'}$, for any $k\neq k'$ \citep{basu2021beyond,kong2022Thompson,liu2020competing,liu2021bandit,sankararaman2021dominate,wang2022bandit}. The preferences of players are uncertain and can be learned through interactive matching iterations.

For each player $p_i$ and arm $a_k\neq a_{k'}$, let $\Delta_{i,k,k'} = \abs{\mu_{i,k} - \mu_{i,k'}}$ be the reward gap between arm $a_k$ and $a_{k'}$ for player $p_i$. Define $\Delta =  \min_{i,k,k'} \Delta_{i,k,k'} > 0$ as the minimum reward gap across all players and arms, which measures the hardness of the learning problem.

\paragraph{Example}
In an online short-term recruitment platform, there are two sides of participants, workers and employers. Workers (arms) are certain about their preferences which are ranked by the payments. Employers (players) are unaware of their preferences over players and this can be stochastic since it could be affected by the exogenous factor.

At each round $t$, every player $p_i$ proposes to an arm $A_i(t) \in \mathcal{K}$. Only the most preferred player will be matched when multiple players select the same arm. Denote $A^{-1}_k(t)$ as the set of players that pull arm $a_k$ at round $t$, then the successfully matched player is $\bar{A}^{-1}_k(t) \in \arg\max_{p_i\in A^{-1}_k(t)} \pi_{k,i}$. Let $\bar{A}_i(t)$ be the matched arm of player $p_i$ at round $t$. Then $\bar{A}_i(t) = A_i(t)$ if player $p_i$ is successfully matched and $\bar{A}_i(t) = \emptyset$ when $p_i$ is rejected. If player $p_i$ is successfully matched, it will receive a random reward $X_i(t)$, which is 1-subgaussian with mean $\mu_{i,\bar{A}_i(t)}$. If player $p_i$ is rejected, it will receive $X_i(t) = 0$. Denote $\bar{A}(t) = \{ (i, \bar{A}_i(t)): i\in [N]\}$ as the matching result at time $t$. Each player $p_i$ can observe the reward $X_i(t)$ and the matching result $\bar{A}(t)$ at time $t$.

Now we introduce the definition of stable matching to define regret formally. Stability is a key property of the matching. A matching $\bar{A}(t)$ is stable if there exists no player-arm pair $(p_i,a_k)$ such that each one prefers each other to its current matched pair, i.e., $\mu_{i,k} > \mu_{i, \bar{A}_i(t)}$ and $\pi_{k,i} > \pi_{k, \bar{A}^{-1}_k(t)}$. A stable matching means that no player or arm is incentive to break the matching. We consider the market satisfies serial dictatorship, where each arm shares the same preference. This condition guarantees unique stable matching \citep{clark2006uniqueness}. Without loss of generality, we assume all arms prefer player from $p_1$ to $p_N$, i.e., $\pi_{k, i}>\pi_{k, i'}, \forall i<i', k\in [K]$. Denote $m^\ast(i)$ as the unique stable matching arm of player $p_i$. The stable regret for each player $p_i$ is defined as the cumulative reward difference between the ideal reward according to the stable matching arm $m^\ast(i)$ and the reward $p_i$ received over $T$ rounds:
\begin{align*}
    R_i(T) = \sum_{t=1}^T \mu_{i,m^\ast(i)} - \EE{\sum_{t=1}^{T}X_i(T)} \,.
\end{align*}

\section{Algorithm}
In this section, We present our multi-level successive selection algorithm (Algorithm \ref{alg:main}). The algorithm is described in Algorithm \ref{alg:main} and can achieve an $O(\frac{N\log(T)}{\Delta^2} + \frac{K\log(T)}{\Delta})$ stable regret.

\begin{algorithm}[htb]
\renewcommand{\algorithmicrequire}{\textbf{Input:}}
\renewcommand\algorithmicensure {\textbf{Output:} }
\caption{Decentralized Multi-Level Successive Selection Framework}
\label{alg:main}
\begin{algorithmic}[1]
\REQUIRE $N,K$.
\STATE Initialize: $\hat{\mu}_{i,k}(0)=0, N_{i,k}(0)=0, \text{UCB}_{i,k}(0)=\infty, \text{LCB}_{i,k}(0)=-\infty$, $\forall i\in[N], k\in [K]$.

\FOR{time $t=1,2,\cdots,N$} \label{alg:index_est_beg}
    \FOR{player $p\in \mathcal{N}$}
        \IF{$p$ has been assigned the index}
            \STATE $p$ pulls $a_2$;
        \ELSE
            \STATE $p$ pulls $a_1$;
        \ENDIF
        \IF{$p$ successfully matches with $a_1$}
            \STATE Assign index $t$ with player $p$.
        \ENDIF
    \ENDFOR
\ENDFOR \label{alg:index_est_end}
\FOR{phase $p=1,2,\cdots$} \label{alg:select_begin}
\STATE $A(i) \leftarrow$\textbf{Communication};
    \FOR{ player $i=1,2,\cdots, N$}
    \STATE Select $A(i)$ for $\log T$ rounds.
    \ENDFOR
\ENDFOR \label{alg:select_end}
\end{algorithmic}
\end{algorithm}

At a high level, if the market satisfies serial dictatorship, then there is a clear hierarchy between players. That is, the player who is most liked by all arms has the highest priority. It does not need to consider the competition with other players when selecting the arm. According to the preference sequence of arms, the level of players decreases in turn. For the player $p_1$, it is actually facing the single-player MAB problem. $p_1$ will run an efficient single-player MAB algorithm to select the best arm, such as the elimination algorithm (Algorithm \ref{alg:eli}) or UCB algorithm (Algorithm \ref{alg:UCB}). For the player $p_2$, avoiding conflicts with $p_1$ when designing the algorithm is necessary.
We can delete the arm pulled by $p_1$ from the plausible arm set of player $p_2$, and player $p_2$ then runs the single-player MAB algorithm in the remaining plausible arm set. Similarly, for the player $p_i$, the single-player MAB algorithm will be carried out outside the arms selected by the player $p_1$, $p_2$, ..., $p_{i-1}$.

Here we introduce the algorithm's process in detail. The first phase proceeds in the first $N$ rounds to assign a distinct index to each player (Line \ref{alg:index_est_beg} - \ref{alg:index_est_end}). At the round $t=1$, all players propose to the arm $a_1$. Moreover, the matched player is assigned with the index $1$. At the round $t=2$, all players except $p_1$ propose to arm $a_1$, and player $p_1$ proposes to arm $a_2$. The player successfully matched with $a_1$ is assigned with index $2$. Similarly, at round $t$, all players that have not been assigned index propose to arm $a_1$, while others propose to arm $a_2$. The player matched with $a_1$ successfully is assigned with index $t$.

When each player has been assigned an index, all players proceed to run an arm selection algorithm (Line \ref{alg:select_begin} - \ref{alg:select_end}). Once player $p_i$ selects arm $A_{i}(t)$ at time $t$ and gets observation $X_i(t)$, he will update the estimated value $\mu_{i,A_i(t)}(t)$ and $N_{i,A_{i}(t)}(t)$ for arm $A_{i}(t)$ as
\begin{align}
    \hat{\mu}_{i,A_{i}(t)}(t) = \frac{ \hat{\mu}_{i,A_{i}(t)}(t-1) \cdot N_{i,A_{i}(t)}(t-1) + X_{i}(t)}{ N_{i,A_{i}(t)}(t-1) + 1}, N_{i,A_{i}(t)}(t) = N_{i,A_{i}(t)}(t-1) + 1 \,.
\end{align}
And $p_i$ lets $\hat{\mu}_{i,k}(t)=\hat{\mu}_{i,k}(t-1)$ and $N_{i,k}(t) = N_{i,k}(t-1)$ for other arms $k\neq A_{i}(t)$. Player $p_i$ then constructs a confidence set for the estimated reward based on historical observation. Specifically, $p_i$ will construct the upper confidence bound UCB index and lower confidence bound LCB index for arm $a_k$ as
\begin{align}
    \text{UCB}_{i,k}(t) = \hat{\mu}_{i,k}(t) + 2\sqrt{\frac{2\log T}{N_{i,k}(t)}} ,
    \text{LCB}_{i,k}(t) = \hat{\mu}_{i,k}(t) - 2\sqrt{\frac{2\log T}{N_{i,k}(t)}} \,.
\end{align}
When $N_{i,k}(t)=0$, we let $\text{UCB}_{i,k}(t) = \infty$ and $\text{LCB}_{i,k}(t) = -\infty$. If there exists two arms $a_k$ and $a_{k'}$ such that $\text{LCB}_{i,k}(t) > \text{UCB}_{i,k'}(t)$, then with high probability $\mu_{i,k} > \mu_{i,k'}$, and $p_i$ can believe that arm $a_k$ is better than $a_{k'}$ with high probability.

Since player $p_1$ ranks highest among all players, it can run the elimination or UCB algorithms for all arms. 

If $p_1$ decides to run the elimination algorithm, it will propose to all arms in a round-robin way to collect enough observation. If for an arm $a_k$ there exists another arm $a_{k'}$ such that $\text{LCB}_{1,k'}(t)>\text{UCB}_{1,k}(t)$ at time $t$, then $p_1$ believes that $a_k$ is not the best arm and thus $p_1$ will eliminate $a_k$ from the plausible set $\mathcal{P}_1$. $p_1$ then proposes to all arms except $a_k$ in a round-robin way. If $p_1$ decides to run the UCB algorithm, it will propose to the arm $A_1(t)\in \arg\max_{k\in \mathcal{P}_1} \text{UCB}_{i.k}(t-1)$ at each time $t$. After each time $t$, $p_1$ will update $\hat{\mu}_{1,k}(t), N_{1,k}(t), \text{LCB}_{1,k}(t)$ and $\text{UCB}_{1,k}(t)$ for any arm $a_k$. Intuitively, from the serial dictatorship, we know that when $p_1$ has identified the best arm, that arm is the stable matched arm for $p_1$.


For player $p_2$, since all arms prefer $p_1$ to $p_2$ and prefer $p_2$ to any other players, thus it is natural for $p_2$ to run the algorithm among arms except the arm $p_1$ will propose to. At each time $t$, $p_2$ will first delete the arm $A_1(t)$ from the plausible set $\mathcal{P}_2$ and then run the single-player algorithm like $p_1$. Intuitively, when $p_1$ has identified its stable matched arm, $p_2$ will find the best arm in the remaining arm set, which is the stable matched arm of player $p_2$.

The algorithm then proceeds similarly for player $p_i$, $\forall i \in [N]$. $p_i$ will first delete arms $A_{1}(t)$, $A_{2}(t)$, $\cdots, A_{i-1}(t)$ from the $\mathcal{P}_i$, and then runs elimination algorithm or UCB algorithm in $\mathcal{P}_i$. When higher-ranked players all select their stable arm, we can conclude that the best arm in $\mathcal{P}_i$ is the stable matched arm for $p_i$. Thus player $p_i$ running single player MAB algorithm in $\mathcal{P}_i$ can find the stable arm eventually.

\textbf{Description for Decentralized Algorithm}
Since in the real world, there would not be a central platform that assigns matching to each player at every time, it is necessary to design a decentralized algorithm where each player selects an arm only based on her history observations. In this section, we describe our decentralized communication framework that helps players complete the multi-level successive selection successfully.

At each $\log T$ round, players would enter the communication process to update their next chosen arm. In this communication block, each player acts as the sender sequentially from $p_1$ to $p_N$, while other players act as receivers. At first, player $p_1$ will first update its chosen arm for the next $\log T$ rounds based on the arm selection procedure (elimination or UCB), then player $p_1$ will check if its chosen arm has changed. If changed, it will select all $K$ arms to cause collisions and thus all receivers will know that they need to start receiving the updated information from the sender. The flag will be set as true, which means the next senders will all send their updated information to followers. If $p_1$ has not changed its chosen arm, it will keep selecting $A_1$ until the communication block ends. Thus other players will not receive information from $p_1$. This ensures no additional regret will occur in the communication process if $p_1$ does not change its chosen arm.

\begin{figure}[th!] 
\centering
\includegraphics[width=1\linewidth]{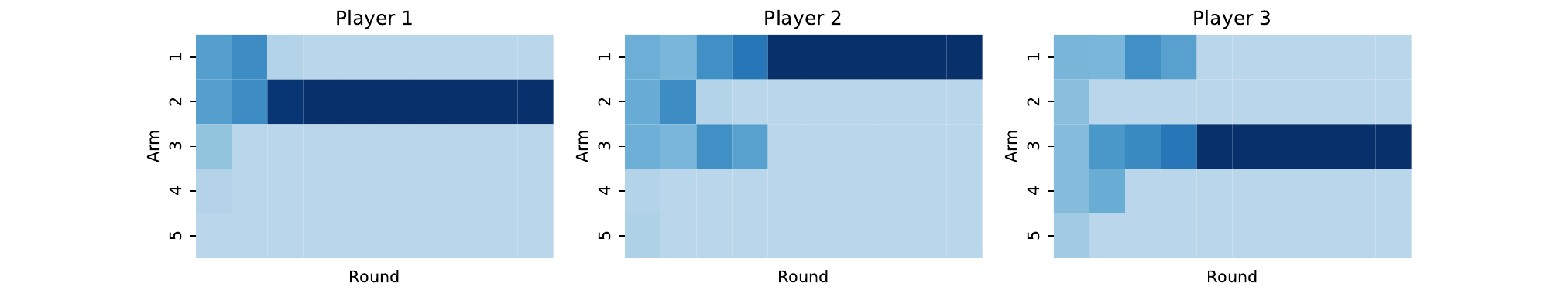}
 \caption{A case of $3$ players and $5$ arms with reward chosen i.i.d. uniform in $[0,1]$. Over $10$ independent runs with horizon $T=100,000$. This heat map counts the number of times the player pulls each arm in every $1,000$ rounds. The color intensity increases with the number of times the arm is selected.}
  \label{fig:1}
\end{figure}

The communication block follows similarly when $p_1$ has sent its message. $p_2$ starts acting as the sender while $p_3$ to $p_N$ acts as the receiver. $p_2$ first updates its chosen arm, and then if the flag has turned true, $p_2$ will directly send its chosen arm to receivers. If the flag is false but $p_2$ has changed its chosen arm, $p_2$ will select $K$ arms sequentially to mention other players, and change the flag to false. Then $p_2$ sends a message to all receivers. This process continues for players $p_3$ to $p_N$, acting as senders one by one.

When $p_i$ acts as sender, $p_{i+1}$ to $p_N$ act receivers. If the flag is true, receivers will directly receive the message from $p_i$, according to the order $i+1,\cdots N$. If the flag is false, they will keep selecting their chosen arm in the last $\log T$ rounds, and they will change the flag and receive a message if any collision occurs.

When the sender wants to send its chosen arm $k$ to the receiver, it will spend $\log K$ rounds to send the binary of $k$. Specifically, it will send the binary of $k$ through collision. The receiver will keep selecting the predefined arm $\log K$ rounds, while the sender selects that arm at $s$ round only when the $s$-bit of $k$ is $1$. Thus the receiver is able to decode the message and update its plausible set.

Figure \ref{fig:1} shows an introductory example of $3$ players and $5$ arms. It can be seen that player $p_1$ first converges to stable matching since it has the highest priority. Then it follows with $p_2$ and $p_3$. This demonstrates the hierarchy process of the algorithm, which motivates the following inductive analysis.

\begin{algorithm}

\renewcommand{\algorithmicrequire}{\textbf{Input:}}

\renewcommand\algorithmicensure {\textbf{Output:} }

\caption{Communication (for player $p_i$)}

\label{alg:communicate}

\begin{algorithmic}[1]
\REQUIRE $N,K$.
\STATE Flag = False, $\mathcal{P}_i = \mathcal{K}$
\FOR{$n=1,2,\cdots, N$}
    \IF{Flag = True}
        \IF{$i < n$}
            \STATE Select $A(i)$;
        \ENDIF
        \IF{$i = n$}
            \STATE $A(i)\leftarrow$ \textbf{Arm Selection Subroutine}($\mathcal{P}_i, N_{i,k}, \text{UCB}_{i,k}, \text{LCB}_{i,k}, \forall k\in[K]$);
            \FOR{$i' = i+1, \cdots, N$} \label{alg:send}
                \STATE \textbf{Send} ($A(i)$, $i'$, $i'$-th arm in $\mathcal{P}_i$);
            \ENDFOR
        \ENDIF
        \IF{$i>n$}
            \STATE \textbf{Receive} ($A(n)$, $n$, $i$-th arm in $\mathcal{P}_i$); \label{alg:receive}
            \STATE $\mathcal{P}_i \leftarrow \mathcal{P}_i \backslash \{ A(n) \}$;
        \ENDIF
    \ENDIF
    \IF{Flag = False}
        \IF{$i < n$}
            \STATE Select $A(i)$;
        \ENDIF
        \IF{$i = n$}
            \STATE $A(i)\leftarrow$ \textbf{Arm Selection Subroutine}($\mathcal{P}_i, N_{i,k}, \text{UCB}_{i,k}, \text{LCB}_{i,k}, \forall k\in[K]$);
            \IF{$A(i)$ has changed}
                \FOR{$k=1,\cdots K$}
                    \STATE Select arm $a_k$;
                \ENDFOR
                \STATE Move to Line \ref{alg:send};
            \ELSE
            \STATE Select $A(i)$;
            \ENDIF
        \ENDIF
        \IF{$i>n$}
            \STATE Select $A(i)$;
            \IF{$A(i)$ gets collide}
                \STATE Flag = True;
                \STATE Move to Line \ref{alg:receive};
            \ENDIF
        \ENDIF
    \ENDIF
\ENDFOR

\end{algorithmic}
\end{algorithm}











\begin{algorithm}[htb] 

\renewcommand{\algorithmicrequire}{\textbf{Input:}}

\renewcommand\algorithmicensure {\textbf{Output:}}

\caption{Elimination Subroutine}

\label{alg:eli}

\begin{algorithmic}[1]

\REQUIRE $\mathcal{P}_i, N_{i,k}(t-1), \text{UCB}_{i,k}(t-1), \text{LCB}_{i,k}(t-1), \forall k\in[K]$.

\FOR{arm $k \in \mathcal{P}_i$}
    \IF{$\exists k'\in\mathcal{P}_i, \text{LCB}_{i,k'}(t-1)> \text{UCB}_{i,k}(t-1)$}
        \STATE $\mathcal{P}_i\leftarrow \mathcal{P}_i\backslash \{k\}$;
    \ENDIF
\ENDFOR
        \STATE Select $A_i(t) \in \arg\min_{k\in\mathcal{P}_i} N_{i,k}(t-1)$;
\ENSURE $A_i(t)$.
\end{algorithmic}
\end{algorithm}

\begin{algorithm}[htb] 

\renewcommand{\algorithmicrequire}{\textbf{Input:}}

\renewcommand\algorithmicensure {\textbf{Output:}}

\caption{UCB Subroutine}

\label{alg:UCB}

\begin{algorithmic}[1]

\REQUIRE $\mathcal{P}_i,  \text{UCB}_{i,k}(t-1), \forall k\in[K]$.

        \STATE Select $A_i(t) \in \arg\max_{k\in\mathcal{P}_i} \text{UCB}_{i,k}(t-1)$.
\ENSURE $A_i(t)$.
\end{algorithmic}
\end{algorithm}

\begin{algorithm}[htb] 
\renewcommand{\algorithmicrequire}{\textbf{Input:}}
\renewcommand\algorithmicensure {\textbf{Output:}}
\caption{Send Message}
\label{alg:sendmessage}
\begin{algorithmic}[1]
\REQUIRE message $m$, receiver $i$, arm $k$.
\STATE Encode $m$ to binary.
\FOR{$\ell = 1, \cdots, {\log K}$}
    \IF{$\ell$-th bit of $m$ is $0$}
        \STATE Select arm $k$;
    \ELSE
    \STATE Skip selection;
    \ENDIF
\ENDFOR

\end{algorithmic}
\end{algorithm}

\begin{algorithm}[htb] 
\renewcommand{\algorithmicrequire}{\textbf{Input:}}
\renewcommand\algorithmicensure {\textbf{Output:}}
\caption{Receive Message}
\label{alg:receivemessage}
\begin{algorithmic}[1]
\REQUIRE Arm $k$.
\STATE Message $m=0$
\FOR{$\ell = 1, \cdots, {\log K}$}
    \STATE Select arm $k$;
    \IF{gets collide}
        \STATE $\ell$-th bit of $m$ $\leftarrow$ $0$;
    \ELSE
        \STATE $\ell$-th bit of $m$ $\leftarrow$ $1$;
    \ENDIF
\ENDFOR
\ENSURE $m$.

\end{algorithmic}
\end{algorithm}

\section{Theoretical Analysis}
In this section, we theoretically analyze our multi-level successive selection algorithm (Algorithm \ref{alg:main}). Before giving the regret guarantee for the algorithm, we first introduce some useful notations. 

 For each player $p_i$, denote $\sigma_i$ as the permutation of arms that represents the preference ranking of $p_i$, i.e., $\mu_{i,\sigma_i(k)} > \mu_{i,\sigma_i(k')}$ for any $k < k'$. Let its player optimal stable matched arm is $\sigma_i(\ell_i)$, $\ell_i\leq N$, which means the stable matched arm of $p_i$ ranks $\ell_i$ in its preference list.

Denote $N_{i,k}(t)$ as the number of successful matchings with $(p_i, a_k)$ up to time $t$. Define $\mathcal{F}=\{\exists i\in[N], \exists k \in [K], \abs{\hat{\mu}_{i,k}(t)-\mu_{i,k}} > \sqrt{\frac{2\log(T)}{N_{i,k}(t)}} \}$ as the bad event that some preferences are not estimated well at time $t$.

We now present the upper bound for the stable regret for each player $p_i \in \mathcal{N}$ by following the multi-level successive selection algorithm with elimination subroutine (MLSE).
\begin{theorem} \label{theorem: 1}
    Following the multi-level successive selection algorithm (Algorithm \ref{alg:main}) with elimination subroutine, the stable regret of player $p_i \in \mathcal{N}$ satisfies
    \begin{align*}
        R_{i}(T) &\leq \frac{32i \log(T)}{\Delta^2} + \sum_{k=\ell_i+1}^{K}  \frac{32\log(T)}{\Delta_{i,\sigma_i(\ell_i),\sigma_i(k)}}  + 2NK + N +  \frac{32 i (N \log K + K)}{\Delta^2}\\
        &\leq   \frac{32i \log(T)}{\Delta^2}+\frac{32(K-\ell_i) \log(T)}{\Delta}+2NK + N +  \frac{32 i (N \log K + K)}{\Delta^2}\,.
    \end{align*}
\end{theorem}

At a high level, the regret $R_i(T)$ is upper bounded by five terms. The first and second terms are due to the selections of sub-optimal arms. Specifically, the second term is the regret induced when stable arm $\sigma_i(\ell_i)$ is in the plausible set, and the first term is induced when $\sigma_i(\ell_i)$ and other higher-ranked arms are deleted in the arm set, i.e., $a_{\sigma_i(1)}$ to $a_{\sigma_i(\ell_i)}$ are pulled by the higher ranked player. Intuitively, if the stable matched arm and other higher-ranked arms are all pulled by higher-ranked players, then $p_i$ will pull the sub-optimal arm even if it has been detected as sub-optimal. The third constant term arises from the bad concentration events $\mathcal{F}_t$. The fourth term is bounded by the regret induced by the rank estimation process. It runs for exact $N$ rounds and thus incurs at most $N$ regret. The last term is caused by communication in the decentralized algorithm.

The primary technical challenge stems from determining how to bound the number of times higher-ranked players pull the stable matched arm of $p_i$, i.e., $\sigma_i(\ell_i)$. Note that $\sigma_i(\ell_{i})$ is the sub-optimal arm for those higher ranked players compared with their stable arm, otherwise $\sigma_i(\ell_i)$ would not be the stable matched arm of $p_i$. Then it is left to bound the number of times $p_1$ to $p_{i-1}$ pulls $\sigma_i(\ell_i)$. For any player $p_{i'}$ with $i'<i$, the number of times $p_{i'}$ pulls $p_i$ is affected by the number of times $p_1$ to $p_{i'-1}$ pull the stable matched arm of $p_{i'}$. This leads to a recursion form of bound.

The following example (Figure \ref{fig:challenge}) shows a case of multi-level blocking. Player $p_3$ may fail to select its stable matched arm $a_{\sigma_3(\ell_3)}$ because $p_1$ or $p_2$ selects $a_{\sigma_3(\ell_3)}$. And $p_2$ may be forced to select $a_{\sigma_3(\ell_3)}$ since $a_{\sigma_2(\ell_2)}$ is selected by $p_1$.
   \begin{figure}[htb]
       \centering
       \includegraphics[width=5.5cm]{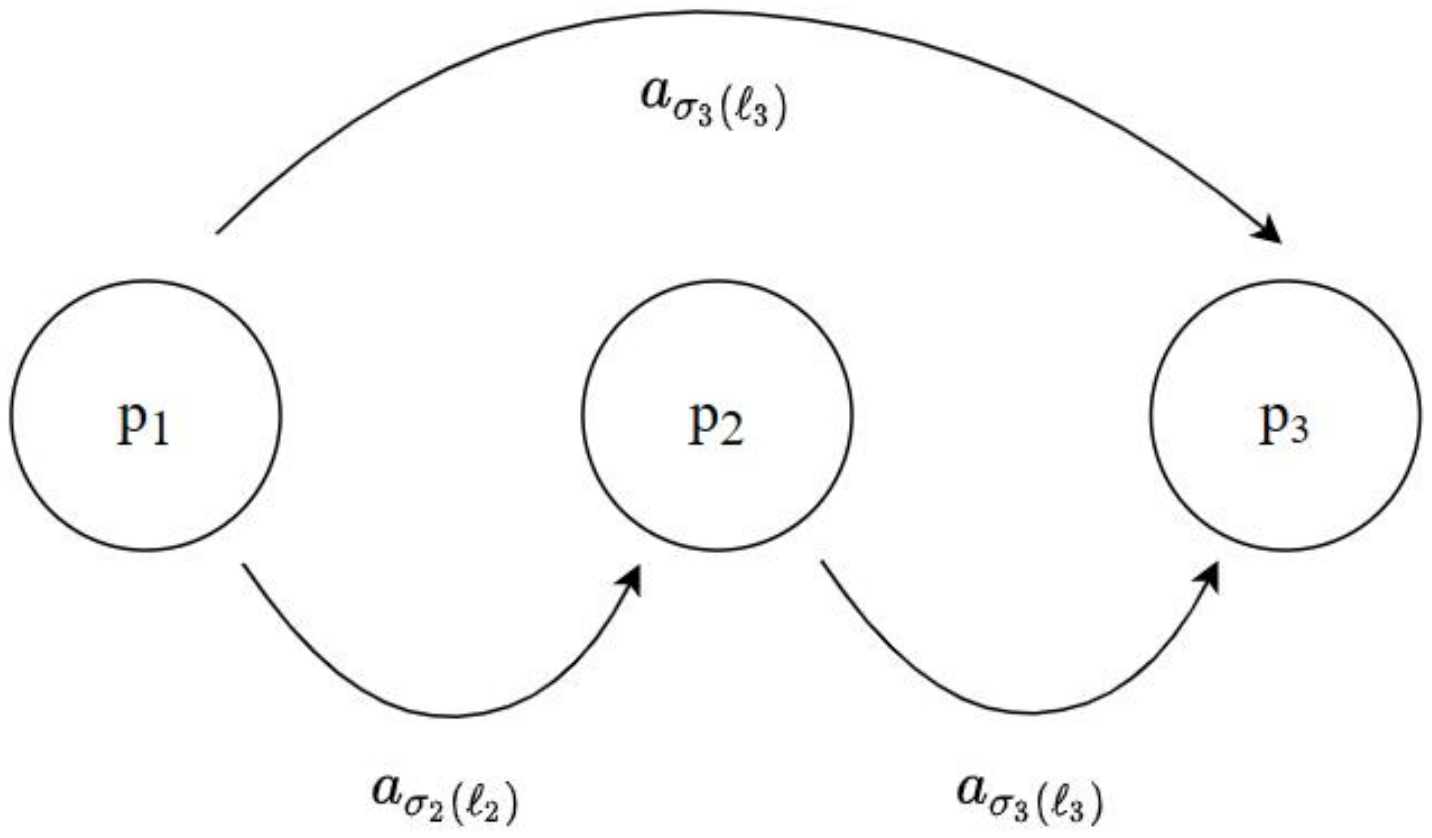}
       \caption{multi-level blocking case}
       \label{fig:challenge}
   \end{figure}

To handle this recursion form, we carefully analyze the process of our algorithm (Algorithm \ref{alg:main}).
A key observation is that when $p_{i'}$ has to select $\sigma_i(\ell_i)$ due to $\sigma_{i'}(\ell_{i'})$ and higher ranked arms are occupied by higher players. This event can date back to the normal exploration of some player $p_{i''}$ with $i''<i'$. Thus the term can be bounded by the summation of exploration times for higher-ranked players.
This helps us prove the bound by induction, which guarantees an $O\left( \frac{(i-1)\log(T)}{\Delta^2} \right) $ bound.

\subsection{Proof of Theorem \ref{theorem: 1}}
Now we give a complete proof of Theorem \ref{theorem: 1}. The analysis for the multi-level selection with UCB sub-routine (ML-UCB) can be found in \ref{lem:ucb} .

The following lemma shows the correctness of the elimination algorithm. 
\begin{lemma} \label{lem:1}
For any player $p_i$ and arm $a_k, a_{k'}$. Conditioned on $\urcorner \mathcal{F}$, $\text{UCB}_{i,k} < \text{LCB}_{i,k'}$ implies $\mu_{i,k} < \mu_{i,k'}$.
\end{lemma}
\begin{proof}
From the definition of LCB and UCB, conditioned on $\urcorner \mathcal{F}$ we have
\begin{align*}
    \text{LCB}_{i,k}(t) = \hat{\mu}_{i,k}(t) - 2\sqrt{\frac{2\log(t)}{N_{i,k}(t)}} \leq \mu_{i,k} \leq \hat{\mu}_{i,k}(t) + 2\sqrt{\frac{2\log(t)}{N_{i,k}(t)}} = \text{UCB}_{i,k}(t) \,.
\end{align*}
Thus $\text{UCB}_{i,k} < \text{LCB}_{i,k'}$ implies that
\begin{align*}
    \mu_{i,k} \leq \text{UCB}_{i,k} < \text{LCB}_{i,k'} \leq \mu_{i,k'} \,.
\end{align*}
Then the lemma is proved.
\end{proof}

This lemma states that the elimination algorithm would mistakenly eliminate the best arm with a very low probability. Then we will show that when all players stop exploration, the player will keep pulling its stable matched arm with high probability.
\begin{lemma} \label{lem:stable}
    If for every player $p_i$, its plausible set $\mathcal{P}_i$ consists of only one arm, i.e., $\abs{\mathcal{P}_i}=1, \forall i\in [N]$, then the arm that each player pulls is its stable matched arm.
\end{lemma}
\begin{proof}
    We prove this lemma by induction. For the first player $p_1$, when $\mathcal{P}_1 = \{ a_k \}$, we have that $a_k=a_{\sigma_1(1)}$ since $p_1$ will not get collided. From the definition of serial dictatorship, we have that $a_{\sigma_1(1)}$ is the stable matched arm for $p_1$.

    Suppose $\abs{\mathcal{P}_i}=1, \forall 1\leq i \leq i-1$, which means players from $p_1$ to $p_{i-1}$ pulls are their stable matched arm. Then consider player $p_i$, assume $\mathcal{P}_i = \{ a_k \}$, then we have that $k\in \arg\max_{k: k\neq \sigma_{i'}(\ell_{i'}), \forall 1\leq i'<i} \mu_{i,k}$. From the serial dictatorship, we have that $a_{k} = a_{\sigma_i(\ell_i)}$. The arm selected by $p_i$ is its stable matched arm.

    Finally, we can conclude that $\abs{\mathcal{P}_i}=1, \forall 1\leq i \leq N$, then the arms from $p_1$ to $p_{N}$ pulls are their stable matched arm. The proof is completed.
\end{proof}



Then we analyze the regret in-depth. Denote $C_i(t)$ as the event that player $i$ is running communication block at time $t$. Regret for player $p_i$ can be and further decomposed as
\begin{align*}
    R_i(T) &= \sum_{k=\ell_i+1}^K \Delta_{i,\sigma_i(\ell_i),\sigma_i(k)} \EE{N_{i,\sigma_{i}(k)}(T)}\\
    &\leq N + \sum_{k=\ell_i+1}^K \sum_{t=1}^{T} \Delta_{i,\sigma_i(\ell_i),\sigma_i(k)}\EE{\mathbbm{1}\{\bar{A}_i(t)=\sigma_i(k), \urcorner C_i(t)\}} + \sum_{t=1}^T \mathbbm{1}\{C_i(t\\
    & \leq N + \sum_{k=\ell_i+1}^K \sum_{t=1}^{T} \Delta_{i,\sigma_i(\ell_i),\sigma_i(k)}\EE{\mathbbm{1}\{\bar{A}_i(t)=\sigma_i(k),\urcorner \mathcal{F},\urcorner C_i(t) \}} \\&+ \sum_{t=1}^T\mathbb{P}\{\mathcal{F} \} +\sum_{t=1}^T \mathbbm{1}\{C_i(t)\} \,.
\end{align*}

Note that each player would never collide in the centralized algorithm, thus we only need to bound the regret due to sub-optimal pulls of player $p_i$.

Define $\mathcal{E}_{i,k}(t):=\{ \forall k^\prime<k, \exists i^\prime<i, \bar{A}_i(t) = \sigma_i(k^\prime) \}$ as the event that arms with reward mean higher than $a_{\sigma_i(k)}$ for $p_i$ are all selected by other higher-ranked players at time $t$. Conditioned on this event, $a_{\sigma_i(k)}$ is the best arm in $\mathcal{P}_i$ for player $p_i$. Then the summation of the number of the sub-optimal pulls $\sum_{k=\ell_i+1}^K N_{i,\sigma_i(k)}(T)$ can be further decomposed as:
\begin{align*}
    &\sum_{k=\ell_i+1}^K \sum_{t=1}^T  \mathbbm{1} \left\{ \bar{A}_{i}(t) = \sigma_i(k),  \urcorner\mathcal{E}_{i,\ell_i+1}(t),\urcorner C_i(t) \right\}\\ &+\sum_{k=\ell_i+1}^K \sum_{t=1}^T  \mathbbm{1} \left\{\bar{A}_{i}(t) = \sigma_i(k), \mathcal{E}_{i,\ell_i+1}(t),\urcorner C_i(t) \right\}\\
    =& \sum_{k=\ell_i+1}^K \sum_{t=1}^T  \mathbbm{1} \left\{ \bar{A}_{i}(t) = \sigma_i(k), \urcorner\mathcal{E}_{i,\ell_i+1}(t),\urcorner C_i(t) \right\} + \sum_{t=1}^T \mathbbm{1} \left\{ \mathcal{E}_{i,\ell_i+1}(t),\urcorner C_i(t) \right\}\,.
\end{align*}

The sub-optimal play only happens in the round-robin elimination phase of the algorithm process. Player $p_i$ pulls sub-optimal arm $a_k$ at time $t$ only when $a_k$ is not eliminated, i.e., $\text{LCB}_{i,\sigma_i(\ell_i)}(t) \leq \text{UCB}_{i,k}(t)$. Here we give a lemma to present the upper bound for the number of observations to identify the better arm.

\begin{lemma}\label{lem:observation_bound}
For any player $p_i$ and arm $a_k, a_{k'}$ with $\mu_{i,k} > \mu_{i,k'}$. Conditioned on $\urcorner \mathcal{F}(t)$, if $\min \{N_{i,k}(t), N_{i,k'}(t)\}>32\log T/\Delta_{i,k,k'}^2$, we have $\text{UCB}_{i,k'}(t) < \text{LCB}_{i,k}(t)$.
\end{lemma}
\begin{proof}
We prove by contradiction. Suppose $\text{UCB}_{i,k'}(t) \geq\text{LCB}_{i,k}(t)$. According to the definition of $\urcorner \mathcal{F}(t)$, LCB and UCB, we have
\begin{align*}
    \mu_{i,k} - 2\sqrt{\frac{2\log(t)}{N_{i,k}(t)}}\leq \text{LCB}_{i,k}(t) \leq \text{UCB}_{i,k'}(t) \leq \mu_{i,k’} + 2\sqrt{\frac{2\log(t)}{N_{i,k‘}(t)}} \,.
\end{align*}
This implies $\Delta_{i,k,k'} = \mu_{i,k} - \mu_{i,k'} \leq 2\sqrt{\frac{2\log(t)}{N_{i,k}(t)}} + 2\sqrt{\frac{2\log(t)}{N_{i,k‘}(t)}} \leq 4\sqrt{\frac{2\log (t)}{\min\{N_{i,k}(t), N_{i,k'}(t) \}}}$, then we can get $\min\{N_{i,k}(t), N_{i,k'}(t)\} \leq \frac{32\log T}{\Delta_{i,k,k'}^2}$. This contradicts with the condition.
\end{proof}

Player $p_i$ runs the elimination algorithm to identify the best arm. The key observation is that the optimal arm is in $\mathcal{P}_i$ when running the elimination algorithm, the expected time of pulling the sub-optimal arm $a_k$ is upper bounded by $32\log (T)/\Delta_{i, \sigma_i(\ell_i), k}^2$. Optimal arm $a_{\sigma_i(\ell_i)}$ is in $\mathcal{P}_i$ at time $t$ only when no player in $\{p_1,p_2,\cdots,p_{i-1} \}$ pulls arm $a_{\sigma_i(\ell_i)}$ at time $t$. Then we have
\begin{align*}
    &\sum_{k>\ell_i} \Delta_{i,\sigma_i(\ell_i),\sigma_i(k)} \sum_{t=1}^T \EE{\mathbbm{1} \left\{\bar{A}_{i}(t) = \sigma_i(k), \urcorner\mathcal{E}_{i,\ell_i+1}(t),\urcorner C_i(t) \right\}}\\ \leq& \sum_{k>\ell_i} \Delta_{i,\sigma_i(\ell_i),\sigma_i(k)} \frac{32\log T}{\Delta_{i,\sigma_i(\ell_i),\sigma_i(k)}^2} \leq \frac{32(K-\ell_i)\log T}{\Delta} \,.
\end{align*}

It is then left to bound the number of times $p_i$ pulls the sub-optimal arm due to the arms with higher rewards selected by higher-ranked players. In other words, we need to bound the number of times the event $ \left\{ \mathcal{E}_{i,\ell_i+1}(t),\urcorner C_i(t) \right\}$ happens:
\begin{align*}
    \EE{\sum_{t=1}^T \mathbbm{1} \left\{ \mathcal{E}_{i,\ell_i+1}(t),\urcorner C_i(t) \right\}}\,.
\end{align*}
This term is related to the number of times player $p_1$ to $p_{i-1}$ pull arm $a_{\sigma_i(\ell_i)}$. Notice that $a_{\sigma_i(\ell_i)}$ is sub-optimal for player $p_1$ to $p_{i-1}$. For some player $p_{i'}$ with $i'<i$, if no player in $p_1$ to $p_{i'-1}$ pulls $a_{\sigma_{i'}(\ell_{i'})}$, then $p_{i'}$ pulls arm $a_i$ no more than $O(\frac{\log(T)}{\Delta^2})$ times, and this leads to $O((i-1)\frac{\log(T)}{\Delta^2})$. Otherwise, $p_{i'}$ may keep pulling $a_{\sigma_i(\ell_i)}$ since $a_{\sigma_{i'}(\ell_{i'})}$ is pulled by another higher ranked player $p_{i''}$. This recursion form motivates us to upper bound the number of times $\left\{ \mathcal{E}_{i,\ell_i+1}(t) \right\}$ happens by induction.

\begin{lemma}\label{lem:higher_block}
    For any $i\in [N]$, the number of times $\left\{ \mathcal{E}_{i,\ell_i+1}(t),\urcorner C_i(t) \right\}$ happens is upper bounded by
    \begin{align*}
        \EE{\sum_{t=1}^T \mathbbm{1} \left\{ \mathcal{E}_{i,\ell_i+1}(t),\urcorner C_i(t) \right\}} \leq \frac{64(i-1)\log(T)}{\Delta^2}\,.
    \end{align*}
\end{lemma}

\begin{proof}
    We prove the following proposition by induction:
    \begin{align*}
        \EE{\sum_{t=1}^T \mathbbm{1} \left\{ \mathcal{E}_{i,\ell_i+1}(t),\urcorner C_i(t) \right\}}
        \leq \sum_{i'=1}^{i-1} \sum_{i''=i'+1}^i \frac{32\log(T)}{\Delta^2_{i',\sigma_{i'}(\ell_{i'}),\sigma_{i''}(\ell_{i''})}}
        \leq \frac{64(i-1)\log(T)}{\Delta^2}\,.
    \end{align*}
    Consider the base case when $i=2$ (when $i=1$, $\left\{ \mathcal{E}_{1,\ell_1+1}(t),\urcorner C_1(t) \right\}$ never happens thus trivially holds), we have that
    \begin{align*}
        \EE{\sum_{t=1}^T \mathbbm{1} \left\{ \mathcal{E}_{2,\ell_2+1}(t),\urcorner C_2(t) \right\}}&\leq \EE{\sum_{t=1}^{T} \mathbbm{1} \left\{ \bar{A}_1(t)= \sigma_2(\ell_2)\right\}}\\
        &\leq \frac{32\log(T)}{\Delta_{1,1,\sigma_2(\ell_2)}^2}\\
        &\leq \frac{32\log(T)}{\Delta^2} \,.
    \end{align*}
    Suppose for player $1\leq i^\prime < i$, $\EE{\sum_{t=1}^T \mathbbm{1} \left\{ \mathcal{E}_{i^\prime,\ell_{i^\prime}+1}(t),\urcorner C_{i'}(t) \right\}} \leq \frac{32(i-1)\log(T)}{\Delta^2}$ holds. Then we analyze the case for player $p_i$.
    
    Recall that $\mathcal{E}_{i,\ell_i+1}(t)=\{ \forall k^\prime<\ell_i+1, \exists i^\prime<i, \bar{A}_i(t) = \sigma_i(k^\prime) \}$. From Lemma \ref{lem:stable}, we have that for arm $\sigma_i(k)$ with $k<\ell_i$, it is another higher ranked player's stable matched arm, i.e., $\exists i^\prime< i, \sigma_{i^\prime}(\ell_{i^\prime}) = \sigma_i(k)$. As for the stable matched arm $\sigma_{i}(\ell_{i})$ for $p_i$, it is the sub-optimal arm for any player $1\leq i^\prime < i$. Then we can bound the term as
    \begin{align*}
        &\EE{\sum_{t=1}^T \mathbbm{1} \left\{ \mathcal{E}_{i,\ell_i+1}(t),\urcorner C_i(t) \right\}}\\
        \leq& \EE{\sum_{i^\prime=1}^{i-1} \sum_{t=1}^T \mathbbm{1} \left\{ \bar{A}_{i^\prime}(t)=\sigma_i(\ell_i), \mathcal{E}_{i,\ell_i+1}(t) \right\}}\\
        \leq& \EE{\sum_{i^\prime=1}^{i-1} \sum_{t=1}^T \mathbbm{1} \left\{ \bar{A}_{i^\prime}(t)=\sigma_i(\ell_i), \mathcal{E}_{i^\prime,\ell_{i^\prime}+1}(t), \mathcal{E}_{i,\ell_i+1}(t) \right\}} \\&+ \EE{\sum_{i^\prime=1}^{i-1} \sum_{t=1}^T \mathbbm{1} \left\{ \bar{A}_{i^\prime}(t)=\sigma_i(\ell_i), \urcorner\mathcal{E}_{i^\prime,\ell_{i^\prime}+1}(t),\mathcal{E}_{i,\ell_i+1}(t) \right\}}\\
        \leq & \EE{ \sum_{i'=1}^{i-2} \sum_{i''=i'+1}^{i-1}\mathbbm{1}\left\{ \bar{A}_{i^\prime}(t)= \sigma_{i''}(\ell_{i''}),\urcorner \mathcal{E}_{i'',\ell_{i''}+1}(t)  \right\} } \\&+ \EE{\sum_{i^\prime=1}^{i-1} \sum_{t=1}^T \mathbbm{1} \left\{ \bar{A}_{i^\prime}(t)=\sigma_i(\ell_i), \urcorner\mathcal{E}_{i^\prime,\ell_{i^\prime}+1}(t),\mathcal{E}_{i,\ell_i+1}(t) \right\}}
        \\
       \leq& \sum_{i'=1}^{i-2} \sum_{i''=i'+1}^{i-1} \frac{32\log(T)}{\Delta^2_{i',\sigma_{i'}(\ell_{i'}),\sigma_{i''}(\ell_{i''})}} +  \sum_{i'=1}^{i-1} \frac{32\log T}{\Delta_{i',\sigma_{i^\prime}(\ell_{i^\prime}),\sigma_i(\ell_i)}^2} \\
       \leq & \sum_{i'=1}^{i-1} \sum_{i''=i'+1}^i \frac{32\log(T)}{\Delta^2_{i',\sigma_{i'}(\ell_{i'}),\sigma_{i''}(\ell_{i''})}}\\
       \leq & \sum_{i'=1}^{i-1} \sum_{j=1}^{i-i'} \frac{32\log(T)}{(j\Delta)^2}\\
       = & \sum_{i'=1}^{i-1} \sum_{j=1}^{i-i'} \frac{1}{j^2} \frac{32\log(T)}{\Delta^2}\\
        \leq& \frac{64(i-1)\log T}{\Delta^2} \,.
   \end{align*}

   The third inequality is from a key observation: when a player $p_{i^\prime}$ is in event $\mathcal{E}_{i^\prime,\ell_{i^\prime}+1}(t)$ at time $t$, then it can only select one arm and thus affect one player at that time. If $\bar{A}_{i^\prime}(t)\neq \sigma_i(\ell_i)$, then at this time $t$ $p_{i^\prime}$ would not make $p_i$ turns to lower ranked arm. If $\bar{A}_{i^\prime}(t)= \sigma_{i''}(\ell_{i''})$ with $i'< i'' < i$, then $p_{i^\prime}$ may make $\mathcal{E}_{i'',\ell_{i''}+1}(t)$ happen and thus it accounts for the number of $\mathcal{E}_{i'',\ell_{i''}+1}(t)$. We consider the worst case, where $p_1$ makes $\mathcal{E}_{2,\ell_{2}+1}(t)$ happen, and then $p_2$ makes $\mathcal{E}_{3,\ell_{3}+1}(t)$ happen, ..., until $p_{i-1}$ makes $\mathcal{E}_{i,\ell_{i}+1}(t)$ happen. Although these events may happen at the same time, only one player could select $a_{\sigma_i(\ell_i)}$ and thus accounts for one time of $\mathcal{E}_{i,\ell_{i}+1}(t)$ happens. Then the number of times $\mathcal{E}_{i,\ell_{i}+1}(t)$ happens is bounded by the summation of sub-optimal pulls by higher ranked players due to exploration.

\end{proof}

\begin{lemma} \label{lem:comm}
    The regret induced by communication is bounded by $O(iN^2K\log K/\Delta^2)$, which is independent of $T$.
\end{lemma}
\begin{proof}
    We first analyze the time cost in one communication block. For player $p_i$ it will at most receive $i-1$ messages from senders $p_1$ to $p_{i-1}$ and send a message to players $p_{i+1},\cdots, p_N$. For receiving messages, player $p_i$ will spend at most $(i-1) \log K$ times receiving from $p_1,\cdots, p_{i-1}$. For sending messages, $p_i$ will spend $(N-i) \log K + (K - i)$ times, where the first term is the cost of sending to $p_{i+1}$ to $p_N$, and the second term is the cost of awaking receivers that communication starts. 
    
    From algorithm design, we know that communication may happen between $\log T$ rounds and any player changes its next $\log T$ rounds' selection. Thus communication will terminate when players have identified their stable arm. Then communication regret for player $p_i$ is bounded by $N  \log K + K$
    multiplying number of times player $1$ to $i$ has identified the stable matched arm, which is
    \begin{align*}
        \frac{32 i (N \log K + K)}{\Delta^2} \,.
    \end{align*}
\end{proof}

Next, we need to bound the probability of the bad event $\cF$ occurring. To do this, we first present a concentration inequality applicable to subgaussian variables.
\begin{lemma} \label{lemma:technical}
(Corollary 5.5 in \cite{lattimore2020bandit}) Assume that $X_1, X_2,\cdots, X_n$ are independent,
$\sigma$-subgaussian random variables centered around $\mu$. Then for any $\epsilon > 0$,
\begin{align*}
    \PP{\frac{1}{n}\sum_{i=1}^n X_i \geq \mu+\epsilon} \leq \exp \left(-\frac{n\epsilon^2}{2\sigma^2} \right) \, , \PP{\frac{1}{n}\sum_{i=1}^n X_i \leq \mu-\epsilon} \leq \exp \left(-\frac{n\epsilon^2}{2\sigma^2} \right) \,.
\end{align*}
\end{lemma}

The following lemma upper bounds the regret caused by the bad event $\mathcal{F}$.
\begin{lemma}\label{lem:bound:fail:event}
\begin{align*}
    \PP{\cF} \le 2NK/T \,.
\end{align*}

\end{lemma}

\begin{proof}
\begin{align*}
    \PP{\cF} &= \PP{ \exists 1 \le t\le T, i\in[N], j\in[K]: \abs{ \hat{\mu}_{i,k}(t) -{\mu}_{i,k}} > \sqrt{\frac{ 6\log T}{ N_{i,k}(t)}} } \\
    &\le \sum_{t=1}^T \sum_{i\in [N]}\sum_{k\in [K]} \PP{ \abs{ \hat{\mu}_{i,k}(t) -{\mu}_{i,k}} > \sqrt{\frac{6 \log T}{ N_{i,k}(t)}  } }\\
    &\le \sum_{t=1}^T \sum_{i\in [N]}\sum_{k\in [K]} \sum_{s=1}^{t} \PP{ T_{i,k}(t)=s, \abs{ \hat{\mu}_{i,k}(t) -{\mu}_{i,k}} > \sqrt{\frac{ 6\log T}{ s }  } }\\
    &\le \sum_{t=1}^T \sum_{i\in [N]}\sum_{k\in [K]} t\cdot 2 \exp(-3\ln T) \\
    &\le 2NK/T \,,
\end{align*} 
where the second last inequality is due to Lemma \ref{lemma:technical}. 
\end{proof}

\begin{corollary}
    If for any player $i\in [N]$, its stable matched arm is its $i$-th favorite arm, then the regret for player $i$ is bounded by
    \begin{align*}
        R_i(T) \leq N + \frac{32 (K-i) \log T}{\Delta} + 2NK + \frac{32 i (N \log K + K)}{\Delta^2}\,.
    \end{align*}
\end{corollary}
\begin{proof}
    For player $p_i$, we have that any arm prefers $i-1$ players to player $p_i$, which implies at least one arm is available among the most preferred $i$ arms for player $i$ at each time. Thus from the standard single-player MAB analysis we have that each sub-optimal arm $k$ with $i<k\leq K$ is pulled at most $32 \log T / \Delta_{i,i,k}^2$ times, which leads to the final result.
\end{proof}

\section{Discussion}


\paragraph{Comparison with centralized UCB algorithm \citep{liu2020competing}} The centralized UCB algorithm proposed in \citet{liu2020competing} runs Gale-Shapley algorithm using players' UCB ranking at each time. It can be seen that when the market satisfies serial dictatorship, the centralized UCB performs same as our multi-level successive selection algorithm with UCB subroutine. However, note that \citet{liu2020competing} only derive an $O\left( \frac{N K \log(T)}{\Delta^2} \right)$ player-pessimal regret upper bound in general market, while we utilize the hierarchy structure and carefully analyze the number of pulls for each sub-optimal arm to obtain the better $O\left( \frac{N \log(T)}{\Delta^2} \right)$ stable regret, which matches the lower bound. Moreover, our decentralized frame work can avoid the assumption that each player knows other players' selections with higher ranks.

\paragraph{The choice of elimination and UCB subroutine} Recall $\mathcal{E}_{i,k}(t):=\{ \forall k^\prime<k, \exists i^\prime<i, \bar{A}_i(t) = \sigma_i(k^\prime) \}$ as the event that arms with higher reward than $a_{\sigma_i(k)}$ for $p_i$ are all selected by other higher-ranked players at time $t$. The key to derive the same regret is that these two algorithms both guarantee 
\begin{align*}
    \sum_{t=1}^T \EE{\mathbbm{1} \left\{\bar{A}_{i}(t) = \sigma_i(k), \urcorner\mathcal{E}_{i,\ell_i+1}(t) \right\}} \leq O\left(\frac{\log T}{\Delta_{i,\sigma_i(\ell_i),\sigma_i(k)}^2}\right) \,.
\end{align*}
When the stable arm and other higher reward arms are not selected by higher-ranked players, then the number of selecting sub-optimal arm $a_{\sigma_i(k)}$ is bounded by $O\left(\frac{\log T}{\Delta_{i,\sigma_i(\ell_i),\sigma_i(k)}^2}\right)$.

\paragraph{Limitations and Future work} This work only focuses on the market satisfying serial dictatorship. The bottleneck in extending our algorithm to other more general settings is that we highly rely on the ranking structure in a serial dictatorship, where players with the higher rank can make decisions without considering lower-rank players. The arm selection procedure is hard to extend to other preference settings. The result does not hold for the more general market. A very interesting future direction is to extend this algorithm or show a tighter lower bound for the general market. The general market does not assume the same preferences of each arm, which means the hierarchy structure among players does not exist. Each player is likely to be rejected. A possible solution is to let each player perform as a leader in turn, and the leader has the top priority to explore. Another direction is to make each player explore and exploit simultaneously to avoid conflict, which may lead to an additional $O(N)$ cost. Then we can integrate the GS algorithm \citep{gale1962college} into our algorithm to find the player-optimal stable matching when the player has identified the preferences.
We leave this as an interesting future work.

\cite{sankararaman2021dominate} studied the alpha-condition setting which is more general than serial dictatorship and is equivalent to the uniqueness condition. Two orders for players and arms respectively are defined, which is possible to be used in our algorithm. We leave it as an interesting future work of extending our algorithm to alpha-condition.

\paragraph{Key Steps Helped Improve upon the Existing Regret Bound} The key step is that we use the definition of $\Delta$ and the property of elimination and UCB subroutine that the arm with a lower reward will be selected a lower number of times. Thus the summation $\sum_k 1 / \Delta_{i,k}^2$ can be bounded by $2 / \Delta^2$. \cite{kong2023player} fail to match the $O(N\log T / \Delta^2)$ lower bound since their algorithm needs to explore each arm the same number of times $O(\log T / \Delta^2)$. Thus they can only get the $O(K\log T / \Delta^2)$ regret. We note that this technique is likely to be applied to the more general setting if the algorithm can guarantee the different explorations for arms with different mean reward. We leave it as an interesting future work.

\section{Conclusion}
In this paper, we study the bandit learning problem in matching markets. We design a round-robin based elimination algorithm for the market with serial dictatorship. The algorithm proceeds from the most preferred player to the least preferred one. Such a design ensures proper exploration for each sub-optimal arm, which is a key idea that avoids regret induced by excessive exploration. Our proposed algorithm attains an $O\left( \frac{N\log(T)}{\Delta^2} + \frac{K\log(T)}{\Delta} \right)$ stable regret bound, which is the first algorithm that matches the lower bound of this problem. We believe that this result is a significant contribution to the problem as it constitutes a step towards closing the gap between the upper bound and lower bound. This work also provides a possible direction when improving the regret bound for the general market.

\section{Acknowledgement}
The corresponding author Shuai Li is supported by National Science and Technology Major Project (2022ZD0114804) and National Natural Science Foundation of China (62376154).
\bibliographystyle{plainnat}
\bibliography{ref.bib}
\appendix

\section{Analysis of ML-UCB} \label{lem:ucb}







The analysis of ML-UCB is similarly to the MLSE algorithm.

We first decompose regret as the selections of sub-optimal arms:
\begin{align*}
    R_i(T) =& \sum_{k=\ell_i+1}^K \Delta_{i,\sigma_i(\ell_i),\sigma_i(k)} \EE{N_{i,\sigma_{i}(k)}(T)}\\
     &\leq N + \sum_{k=\ell_i+1}^K \sum_{t=1}^{T} \Delta_{i,\sigma_i(\ell_i),\sigma_i(k)}\EE{\mathbbm{1}\{\bar{A}_i(t)=\sigma_i(k), \urcorner C_i(t)\}} + \sum_{t=1}^T \mathbbm{1}\{C_i(t)\}\\
     &\leq N + \sum_{k=\ell_i+1}^K \sum_{t=1}^{T} \Delta_{i,\sigma_i(\ell_i),\sigma_i(k)}\EE{\mathbbm{1}\{\bar{A}_i(t)=\sigma_i(k),\urcorner \mathcal{F}(t),  \urcorner C_i(t)\}} \\ & + \EE{\sum_{t=1}^T\mathbbm{1} \{\mathcal{F}_t \}} + \sum_{t=1}^T \mathbbm{1}\{C_i(t)\} \,.
\end{align*}

For the number of sub-optimal selections $\sum_{t=1}^T \mathbbm{1}\{\bar{A}_i(t)=\sigma_i(k),\urcorner C_i(t)\}$, we have
\begin{align*}
    &\sum_{k=\ell_i+1}^K \sum_{t=1}^T \mathbbm{1} \left\{ \bar{A}_{i}(t) = \sigma_i(k), \urcorner C_i(t) \right\}\\
    =&\sum_{k=\ell_i+1}^K \sum_{t=1}^T  \mathbbm{1} \left\{ \bar{A}_{i}(t) = \sigma_i(k),  \urcorner\mathcal{E}_{i,\ell_i+1}(t) \right\} +\sum_{k=\ell_i+1}^K \sum_{t=1}^T  \mathbbm{1} \left\{\bar{A}_{i}(t) = \sigma_i(k), \mathcal{E}_{i,\ell_i+1}(t) \right\}\\
    =& \sum_{k=\ell_i+1}^K \sum_{t=1}^T  \mathbbm{1} \left\{ \bar{A}_{i}(t) = \sigma_i(k),  \urcorner\mathcal{E}_{i,\ell_i+1}(t) \right\} + \sum_{t=1}^T\sum_{k=\ell_i+1}^K  \mathbbm{1} \left\{\bar{A}_{i}(t) = \sigma_i(k),\mathcal{E}_{i,\ell_i+1}(t) \right\}\\
    =& \sum_{k=\ell_i+1}^K \sum_{t=1}^T  \mathbbm{1} \left\{ \bar{A}_{i}(t) = \sigma_i(k), \urcorner\mathcal{E}_{i,\ell_i+1}(t) \right\} + \sum_{t=1}^T \mathbbm{1} \left\{ \mathcal{E}_{i,\ell_i+1}(t) \right\}\,.
\end{align*}

For the first term, we utilize the standard UCB analysis that each sub-optimal arm is selected at most $\frac{32\log T}{\Delta_{i,\sigma_i(\ell_i),\sigma_i(k)}^2}$ when the stable arm $\sigma_i(\ell_i)$ is in the plausible set $\mathcal{P}_i$:
\begin{align*}
    &\sum_{k>\ell_i} \Delta_{i,\sigma_i(\ell_i),\sigma_i(k)} \sum_{t=1}^T \EE{\mathbbm{1} \left\{\bar{A}_{i}(t) = \sigma_i(k), \urcorner\mathcal{E}_{i,\ell_i+1}(t) \right\}}\\ \leq& \sum_{k>\ell_i} \Delta_{i,\sigma_i(\ell_i),\sigma_i(k)} \frac{32\log T}{\Delta_{i,\sigma_i(\ell_i),\sigma_i(k)}^2}\\
    =&\sum_{k>\ell_i}\frac{32\log T}{\Delta_{i,\sigma_i(\ell_i),\sigma_i(k)}} \\ \leq& \frac{32(K-\ell_i)\log T}{\Delta} \,.
\end{align*}

For the second term, similar to the analysis of MLSE, the number of times arm $\sigma_i(\ell_i)$ selected by higher ranked players is upper bounded by their sub-optimal pulls when their stable arms are not selected.
\begin{align*}
        &\EE{\sum_{t=1}^T \mathbbm{1} \left\{ \mathcal{E}_{i,\ell_i+1}(t) \right\}}\\
        \leq& \EE{\sum_{i^\prime=1}^{i-1} \sum_{t=1}^T \mathbbm{1} \left\{ \bar{A}_{i^\prime}(t)=\sigma_i(\ell_i), \mathcal{E}_{i,\ell_i+1}(t) \right\}}\\
        \leq& \EE{\sum_{i^\prime=1}^{i-1} \sum_{t=1}^T \mathbbm{1} \left\{ \bar{A}_{i^\prime}(t)=\sigma_i(\ell_i), \mathcal{E}_{i^\prime,\ell_{i^\prime}+1}(t), \mathcal{E}_{i,\ell_i+1}(t) \right\}} \\&+ \EE{\sum_{i^\prime=1}^{i-1} \sum_{t=1}^T \mathbbm{1} \left\{ \bar{A}_{i^\prime}(t)=\sigma_i(\ell_i), \urcorner\mathcal{E}_{i^\prime,\ell_{i^\prime}+1}(t),\mathcal{E}_{i,\ell_i+1}(t) \right\}}\\
        \leq & \EE{ \sum_{i'=1}^{i-2} \sum_{i''=i'+1}^{i-1}\mathbbm{1}\left\{ \bar{A}_{i^\prime}(t)= \sigma_{i''}(\ell_{i''}),\urcorner \mathcal{E}_{i'',\ell_{i''}+1}(t)  \right\} } \\&+ \EE{\sum_{i^\prime=1}^{i-1} \sum_{t=1}^T \mathbbm{1} \left\{ \bar{A}_{i^\prime}(t)=\sigma_i(\ell_i), \urcorner\mathcal{E}_{i^\prime,\ell_{i^\prime}+1}(t),\mathcal{E}_{i,\ell_i+1}(t) \right\}}
        \\
       \leq& \sum_{i'=1}^{i-2} \sum_{i''=i'+1}^{i-1} \frac{32\log(T)}{\Delta^2_{i',\sigma_{i'}(\ell_{i'}),\sigma_{i''}(\ell_{i''})}} +  \sum_{i'=1}^{i-1} \frac{32\log T}{\Delta_{i',\sigma_{i^\prime}(\ell_{i^\prime}),\sigma_i(\ell_i)}^2} \\
       \leq & \sum_{i'=1}^{i-1} \sum_{i''=i'+1}^i \frac{32\log(T)}{\Delta^2_{i',\sigma_{i'}(\ell_{i'}),\sigma_{i''}(\ell_{i''})}}\\
       \leq & \sum_{i'=1}^{i-1} \sum_{j=1}^{i-i'} \frac{32\log(T)}{(j\Delta)^2}\\
       = & \sum_{i'=1}^{i-1} \sum_{j=1}^{i-i'} \frac{1}{j^2} \frac{32\log(T)}{\Delta^2}\\
        \leq& \frac{64(i-1)\log T}{\Delta^2} \,.
   \end{align*}

For number of communications we have from Lemma \ref{lem:comm} that $\sum_{t=1}^T{\bOne{C_i(t)}}$ is bounded by
\begin{align*}
        \frac{32 i (N \log K + K)}{\Delta^2} \,.
\end{align*}

   Then we can get the same $O\left( \frac{N\log T}{\Delta^2} + \frac{K\log T}{\Delta} \right)$ regret.








\end{document}